\documentclass[11pt]{article}

\usepackage[margin=1in]{geometry}
\usepackage{amsmath,amssymb,amsthm}
\usepackage{booktabs}
\usepackage{array}
\usepackage{enumitem}
\usepackage[hidelinks]{hyperref}
\usepackage{microtype}
\usepackage{xcolor}
\usepackage{tikz}
\usetikzlibrary{arrows.meta,positioning}

\title{Task-to-Model Optimization for Enterprise LLM Coding Assistants:\
A Data-Driven Framework for Cost-Optimal Routing}

\author{
\small Srinivasan Manoharan \quad Junhua Zhao \quad Fangbo Tu \quad Haifeng Wu \quad Jian Wan \\[2pt]
\small Maliah Rajan M \quad Ashwin Hegde \quad Mithun Sasidharan \quad Kalyan Chakravarthi Podamekala \\[2pt]
\scriptsize\texttt{\{srinivmanoharan, zhahua, fatu, haifwu, jwan, mrajanm, ashegde, msasidharan, kpodamekala\}@paypal.com}
}
\date{}

\newcommand{\method}{\textsc{T2MO}}
\newcommand{\clang}{\textsc{Clio}}

\newtheorem{proposition}{Proposition}

\begin{document}
\maketitle

\begin{abstract}
Enterprise AI coding assistants incur substantial inference spend, and naive token-cost minimization often fails to reduce end-to-end cost once retries, escalations, and developer wait time are included. We present Task-to-Model Optimization (\method{}), a data-driven methodology for optimizing model selection in production coding workflows. We treat each developer session as a task that can be discovered, classified, graded for difficulty, benchmarked in a production-like harness, and routed to the cheapest model able to complete it within quality and latency constraints. The framework is a nine-stage pipeline spanning telemetry instrumentation, taxonomy discovery, difficulty grading, benchmark construction, candidate evaluation, optimal mix derivation, forecasting and version planning, staged routing deployment, and continuous governance. Unlike token-centric routing rules, our objective is cost per completed task, with failure escalation priced in explicitly. We show that this expected-completion-cost objective weakly dominates token-cost minimization under escalation, and we derive the routing boundary, the minimum pass rate a cheaper model must reach on a given cell to be worth deploying. Decisions are organized as a two-level hierarchy of task category $\times$ difficulty tier, and per-cell displacement opportunities are aggregated into a traffic-weighted savings waterfall that ranks replacement candidates by realized dollar impact. The framework supports developer guidance, spend forecasting, and a staged transition from static policies to shadow-mode classifiers, verified cascades, and ultimately an intelligent router. We describe the methodology, optimization objective, evaluation protocol, and governance loop in a form suitable for production deployment and future empirical study.
\end{abstract}

\section{Introduction}
Large-scale deployment of AI coding assistants can lead to rapid and highly concentrated inference spend. In the enterprise setting that motivates this work, monthly spend on the coding assistant is in the millions of dollars, and about 98\% of it is concentrated in two frontier-class models. When spend is this concentrated, meaningful savings require shifting part of the workload to cheaper models, including lower-tier commercial models and open-source models served on owned or committed GPU capacity, without degrading task completion quality or developer experience.

This paper proposes \method{}, a framework for making such displacement decisions from evidence rather than intuition. Its central thesis is that model routing should be posed at the level of the \emph{task}, not the prompt or the token. Prior routing and cascade methods choose between a strong and a weak model per query, usually to minimize token cost or query difficulty~\cite{frugalgpt,routellm,hybridllm,automix}. We instead decide displacement per homogeneous task-category $\times$ difficulty cell, accounting for the probability of task completion, the cost of escalation when a candidate fails, and the cost of developer waiting and context switching. Seen this way, a model with a lower token price can still be more expensive end-to-end if it fails often enough to force retries on a stronger incumbent.

The methodology is designed to serve three goals. First, it produces per-subcategory, per-difficulty recommendations mapping tasks to the cheapest sufficient model. Second, it generates a rolling twelve-month forecast that decomposes spend into adoption, activity, mix, and price, while treating model releases and pricing changes as scenario events. Third, it provides a staged path from static policy to a production intelligent router.

\paragraph{Contributions.} This paper makes the following contributions.
\begin{enumerate}[leftmargin=1.4em,itemsep=1pt,topsep=2pt]
\item A \emph{task-aware} enterprise routing framework that decomposes production traffic into a two-level hierarchy of task category $\times$ difficulty tier, so that displacement decisions are made per homogeneous cell rather than per prompt or per token (Section~\ref{sec:hierarchy}).
\item An \emph{expected completion cost} objective that prices retries, escalation to the incumbent, and developer wait time explicitly, and a proof that under escalation this objective dominates naive token-cost minimization (Section~\ref{sec:problem}).
\item A \emph{traffic-weighted savings waterfall} that turns per-cell displacement decisions into an auditable, dollar-ranked backlog of replacement opportunities, illustrated on the two highest-value categories in our environment (Section~\ref{sec:hierarchy}).
\item A distribution-matched, in-harness benchmark methodology and a staged deployment path from static policy to shadow classifier, verified cascade, and intelligent router (Sections~\ref{sec:benchmark}--\ref{sec:router}).
\end{enumerate}

\section{Related Work}
\label{sec:related}
\paragraph{Cost-aware LLM routing and cascades.}
The idea of serving each query with the cheapest adequate model has been explored from several angles. FrugalGPT~\cite{frugalgpt} introduces LLM cascades that query models in increasing order of cost and stop once a learned scorer deems an answer acceptable, reporting large cost reductions at matched quality. AutoMix~\cite{automix} adds few-shot self-verification and a POMDP-based meta-router that escalates to a larger model only when a smaller model's output is judged unreliable. RouteLLM~\cite{routellm} learns binary routers between a strong and a weak model from preference data, and Hybrid LLM~\cite{hybridllm} routes on predicted query difficulty against a tunable quality target. Shnitzer et al.~\cite{shnitzer} repurpose benchmark datasets to learn per-task model selectors. Our objective, expected cost per \emph{completed} task with escalation priced in (Section~\ref{sec:problem}), generalizes the two-model cascade to a two-level task-category $\times$ difficulty hierarchy and closes the loop with a traffic-weighted savings waterfall that ranks displacement opportunities by realized dollar impact.

\paragraph{Task taxonomy and difficulty.}
Our taxonomy discovery follows the bottom-up clustering methodology of \clang{}~\cite{clio}, which summarizes and clusters production conversations into an interpretable hierarchy. We extend this by weighting nodes by share of cost rather than requests and by grading each cell's difficulty with a validated judge.

\paragraph{Evaluation and judges.}
Cell-level evaluation relies on production-like benchmarks and automatic grading. We build on code benchmarks such as HumanEval~\cite{humaneval} for self-contained generation and SWE-bench~\cite{swebench} for repository-scale issue resolution, and on the LLM-as-judge protocol validated by MT-Bench~\cite{mtbench}, whose calibration-and-bias findings motivate our verifier-first grading and quarterly recalibration.

\section{Method at a Glance}
\method{} is implemented as a nine-phase pipeline. The first six phases derive the ideal model mix, while the final three operationalize the result into forecasting, routing, and governance. The pipeline is intentionally iterative: taxonomy, benchmarks, and mix assignments are refreshed on a defined cadence because traffic, models, and prices drift over time.

\begin{table}[t]
\centering
\begin{tabular}{>{\raggedright\arraybackslash}p{0.25\linewidth} >{\raggedright\arraybackslash}p{0.46\linewidth} >{\raggedright\arraybackslash}p{0.18\linewidth}}
\toprule
\textbf{Phase} & \textbf{Output} & \textbf{Cadence} \\
\midrule
1. Instrumentation & Complete session telemetry & Continuous \\
2. Taxonomy discovery & Hierarchical task taxonomy weighted by cost & Monthly \\
3. Difficulty grading & Calibrated easy/medium/difficult labels & Quarterly \\
4. Benchmark construction & Distribution-matched eval sets in production harness & Per release + quarterly \\
5. Candidate evaluation & Pass rate, cost, latency by cell & Per release \\
6. Mix derivation & Cost-optimal model share by cell & Monthly \\
7. Forecasting & Rolling 12-month spend forecast & Monthly \\
8. Routing & Static policy, shadow mode, cascade, router & Staged rollout \\
9. Governance & KPI dashboard, guardrails, ownership & Monthly review \\
\bottomrule
\end{tabular}
\caption{The \method{} pipeline. Each stage is refreshed on a schedule because both traffic composition and model economics evolve.}
\label{tab:pipeline}
\end{table}

\section{Problem Setting}
\label{sec:problem}
We observe developer sessions containing prompts, tool calls, token counts, latency, and outcome signals. The objective is to assign each task cell $c$ to a model $m$ such that expected end-to-end cost is minimized subject to quality, latency, and eligibility constraints.

Let $C_m(c)$ denote the average in-harness cost for model $m$ on cell $c$, $P_m(c)$ the verifier pass rate, $C_M(c)$ the cost of the incumbent model, and $W$ the developer wait and context-switch cost. The expected cost of routing cell $c$ to model $m$ is
\begin{equation}
\mathbb{E}[\mathrm{cost}\mid m,c] = C_m(c) + (1 - P_m(c))\cdot [C_M(c) + W].
\label{eq:expected_cost}
\end{equation}
Displacement is justified only when the expected cost under the candidate is lower than the incumbent cost and the quality constraint is satisfied. This objective is materially different from token-cost minimization because retries and escalations are explicitly priced.

The following observation makes precise why the expected-completion-cost objective is the safer decision rule. Token-cost minimization displaces the incumbent whenever the candidate's in-harness cost is lower, i.e.\ $C_m(c) < C_M(c)$, ignoring the pass rate entirely.

\begin{proposition}[Escalation dominance]
\label{prop:dominance}
Fix a cell $c$ and a candidate model $m$. Under the expected-completion-cost rule of Eq.~\eqref{eq:expected_cost}, $m$ displaces the incumbent $M$ only if
\begin{equation}
P_m(c) \;>\; 1 - \frac{C_M(c) - C_m(c)}{C_M(c) + W}.
\label{eq:boundary}
\end{equation}
Consequently, every cell displaced under the expected-completion-cost rule is also displaced under token-cost minimization, but not conversely; and the realized expected cost of the expected-completion-cost policy never exceeds that of the token-cost policy.
\end{proposition}

\begin{proof}
Since $\mathbb{E}[\mathrm{cost}\mid m,c] = C_m(c) + (1-P_m(c))(C_M(c)+W)$, the displacement condition $\mathbb{E}[\mathrm{cost}\mid m,c] < C_M(c)$ rearranges directly to Eq.~\eqref{eq:boundary}. Because $(1-P_m(c))(C_M(c)+W)\ge 0$, we have $\mathbb{E}[\mathrm{cost}\mid m,c]\ge C_m(c)$, so $\mathbb{E}[\mathrm{cost}\mid m,c] < C_M(c)$ implies $C_m(c) < C_M(c)$: the expected-cost rule is strictly more conservative. When the token-cost rule displaces a cell with pass rate below the threshold in Eq.~\eqref{eq:boundary}, the candidate's realized expected cost exceeds $C_M(c)$, whereas the expected-cost rule keeps the incumbent; on all other cells the two rules agree. Hence the expected-cost policy weakly dominates.
\end{proof}

The right-hand side of Eq.~\eqref{eq:boundary} is the \emph{routing boundary}: the minimum pass rate a candidate must achieve on a cell to be worth deploying, given the token-cost gap it opens and the escalation penalty it risks. Cheaper candidates (larger $C_M(c)-C_m(c)$) and cheaper escalation (smaller $W$) lower the bar; expensive interruptions raise it.

\section{Telemetry and Data Foundation}
All downstream decisions depend on complete and faithful session logging. Each session should capture identity and surface information, interaction history, economics, and outcome signals. The required telemetry includes pseudonymized developer identity, calling framework, prompt/response transcript or privacy-safe summary, turn count, tool-call sequence, files and repositories touched, diff size, model identity, input/output/cache tokens, unit prices at time of call, computed cost, wall-clock latency, time-to-first-token, and outcome signals such as test results, compile or lint exit codes, reprompting, abandonment, redo on the same artifact, PR merge, PR revert, and review cycle count.

Outcome signals are especially important because benchmark pass/fail is only a proxy. Real-world re-prompt rate, abandonment, and revert rate are the ground truth that validates whether a routing policy improves or degrades developer experience.

Data quality gates are necessary before using the telemetry for policy. A reasonable target is fewer than 5\% unclassified traffic, monthly cost reconciliation within 1\% of the cloud inference bill, and privacy review for transcript retention. Where possible, Clio-style summaries should be retained rather than raw prompts.

\section{Task Taxonomy Discovery}
\method{} uses a bottom-up discovery process inspired by the \clang{} methodology~\cite{clio}. Sessions are decomposed into topic-based subtasks, summarized, and semantically clustered to discover the task taxonomy directly from production traffic. The resulting hierarchy runs from calling framework at the top level to root category and subcategory below, and each node is weighted by share of total cost rather than share of requests.

Weighting by cost ensures that optimization effort concentrates where the dollars are. In the motivating environment, Claude Code CLI represents the majority of traffic, while subcategories such as software development, AI tooling and documentation, testing and QA, and Git workflow contribute disproportionately to spend and therefore deserve dedicated benchmark coverage.

The taxonomy is versioned. Monthly reclustering compares the new taxonomy against the previous month. Any newly emerging cluster above a cost threshold should trigger benchmark coverage, while clusters that shrink below threshold can be retired from active benchmarking but should retain routing policy for continuity.

\section{Difficulty Grading and Judge Validation}
Task type alone is insufficient for routing because difficulty varies within each subcategory. Some models may achieve near-perfect performance on easy and medium tasks while failing on difficult tasks. Consequently, \method{} distinguishes between two difficulty problems.

\paragraph{Ex post grading.}
An LLM judge labels completed sessions using the full transcript. This provides the ground truth difficulty labels for benchmark construction.

\paragraph{Ex ante prediction.}
A router requires a difficulty estimate before the task begins. This is harder because easy and hard tasks can appear similar in the opening prompt. Ex ante difficulty therefore relies on observable features such as repository size and language mix, number of files plausibly in scope, presence of stack traces or error logs, prompt length and specificity, cross-cutting concerns such as concurrency or migrations, historical difficulty of the same developer/repository/subcategory, and time-of-session context.

Before difficulty labels are trusted, the judge must be calibrated. A practical gate is to sample at least 300 sessions stratified across the taxonomy, collect human labels using a written rubric, and require human-human Cohen's kappa of at least 0.6 and judge-human agreement of at least 0.7. When grading candidate model outputs, the methodology prefers neutral rubric-based judging and, where available, objective verifiers such as tests passing, builds compiling, and PRs merging without revert. Recalibration should occur quarterly and whenever the judge model changes.

\section{Benchmark Construction}
\label{sec:benchmark}
Benchmarks are sampled in proportion to the real traffic distribution, weighted by cost, and stratified by difficulty within each subcategory. This ensures that aggregate benchmark performance predicts aggregate production impact. Public benchmarks are useful for generalization studies, but they do not reflect a specific enterprise task mix and therefore cannot directly support displacement decisions.

Fidelity of the evaluation harness matters as much as the sample itself. Candidate models must be evaluated inside the same agentic scaffold that developers use in production, including the same tool suite, multi-turn behavior, and system prompts. Single-turn raw API calls are not adequate substitutes, because pass rates measured outside the harness may not transfer.

Wherever possible, automated verifiers should replace human judging. For testing and QA the verifier can be unit tests; for Git workflow, git-state assertions; for refactoring, compile, lint, and typecheck; and for frontend tasks, rendered-output checks. Categories with cheap verifiers are natural candidates for cascade routing, since failure can be detected automatically before the developer sees it.

Sample size also governs how confidently a cell can be decided. As a working rule, each subcategory-by-difficulty cell that will support a confirmed routing decision should contain at least 30 tasks, with binomial confidence intervals reported on pass rates. The decision variable is the pass-rate delta between candidate and incumbent on the same tasks, priced through the escalation model, not the candidate's absolute distance from perfection.

\section{Candidate Model Evaluation}
Each candidate model is evaluated per subcategory-by-difficulty cell on four dimensions: task completion, true cost, latency, and operational fit. Task completion is measured by verifier pass rate, or by judge-graded rubric score where no verifier exists. True cost is the in-harness cost of completing the task. Latency includes end-to-end latency and time-to-first-token, which both matter in interactive settings. Operational fit covers context-window sufficiency, tool-calling reliability, deployment surface, and compliance or legal clearance.

Token price alone is not a sufficient decision criterion. Cheaper models can consume more turns and more tokens per task, and their failures may trigger a second execution on the stronger incumbent plus developer wait time. Equation~\ref{eq:expected_cost} captures this trade-off explicitly.

In practice the candidate set is not a single cheap model versus a single incumbent but an ordered roster of $n$ models spanning several cost tiers---in our current benchmarking pass, GLM-5.2, Kimi K3, Sonnet 4.6, Sonnet 5, and Opus 4.8, listed in increasing order of cost per task. Every candidate in the roster is benchmarked on every cell; a model \emph{clears} a cell if it meets the quality floor and stays under the latency ceiling on that cell's tasks. Section~\ref{sec:mix} generalizes the two-tier displacement rule of Eq.~\eqref{eq:boundary} to this $n$-ary roster.

\section{Deriving the Ideal Model Mix}
\label{sec:mix}
Given per-cell evaluation results, the ideal mix is a constrained assignment problem. For each cell, choose the model minimizing expected cost per completed task subject to a quality floor, a latency ceiling, and deployment eligibility. Aggregating the per-cell assignments by traffic weight yields the target model mix, and multiplying each reassigned cell's traffic weight by its per-task savings yields the addressable-savings waterfall.

\paragraph{$n$-ary cascade assignment.} With the ordered candidate roster of Section~\ref{sec:benchmark}--\ref{sec:hierarchy} ($m_1 \prec m_2 \prec \cdots \prec m_n$ by increasing cost), the per-cell rule generalizes Eq.~\eqref{eq:boundary} directly: a cell is \emph{assigned} to the cheapest candidate that clears its quality floor and latency ceiling, and it \emph{escalates}, on failure, to the next candidate up the roster that clears the cell---not necessarily the global frontier incumbent. This keeps the escalation cost $W$ local: a GLM-5.2 failure that Kimi K3 or Sonnet~4.6 can already resolve should not be priced as an escalation all the way to Opus~4.8. Table~\ref{tab:cascade} illustrates the rule on five complexity-graded task groups spanning the full candidate roster.

\begin{table}[t]
\centering
\footnotesize
\begin{tabular}{l l r l l l}
\toprule
\textbf{Task group} & \textbf{Complexity} & \textbf{Traffic} & \textbf{Clears floor} & \textbf{Assigned} & \textbf{Escalates to} \\
\midrule
Boilerplate \& unit tests      & Easy        & 22\% & GLM-5.2, Kimi K3, Sonnet 4.6 & GLM-5.2   & Sonnet 4.6 \\
Code review \& small fixes     & Easy--Med.\ & 18\% & Kimi K3, Sonnet 4.6, Sonnet 5 & Kimi K3   & Sonnet 5 \\
Feature implementation         & Medium      & 30\% & Sonnet 4.6, Sonnet 5          & Sonnet 4.6 & Sonnet 5 \\
Multi-file refactors           & Hard        & 8\%  & Sonnet 5, Opus 4.8            & Sonnet 5  & Opus 4.8 \\
Architecture \& deep debugging & Hard        & 5\%  & Opus 4.8 only                 & Opus 4.8  & --- \\
\bottomrule
\end{tabular}
\caption{Illustrative $n$-ary cascade assignment across five complexity-graded task groups (candidates listed left to right by increasing cost: GLM-5.2, Kimi K3, Sonnet 4.6, Sonnet 5, Opus 4.8). Each row is assigned the cheapest candidate that clears the quality floor and latency ceiling for that group; on failure it escalates to the next candidate up the roster that clears it, not directly to the frontier model. Traffic shares and cell assignments are illustrative placeholders pending full benchmark completion; the remaining $17\%$ of traffic sits in smaller task groups aggregated the same way.}
\label{tab:cascade}
\end{table}

Aggregating cascade assignments into a company-wide mix follows the same traffic-weighting logic as Eq.~\eqref{eq:cellweight}: a model's share is the sum of the traffic of the cells it wins, net of the share that escalates away from it. For example, if GLM-5.2 is assigned to cells covering $22\%$ of traffic and roughly $10\%$ of that traffic escalates to Sonnet~4.6, GLM-5.2's net share falls to about $19.8\%$ while Sonnet~4.6 gains the corresponding $2.2\%$. Table~\ref{tab:targetmix} reports the resulting illustrative target mix once this aggregation and escalation adjustment are applied across the full taxonomy; it is re-derived monthly as traffic and models shift, and it is what per-model budgets, spend caps, and router policy are sized against.

\begin{table}[t]
\centering
\begin{tabular}{l r}
\toprule
\textbf{Model} & \textbf{Illustrative target share of tasks} \\
\midrule
GLM-5.2   & 38\% \\
Sonnet 5  & 34\% \\
Kimi K3   & 12\% \\
Haiku     & 11\% \\
Opus 4.8  & 5\%  \\
\bottomrule
\end{tabular}
\caption{Illustrative company-wide target model mix, aggregated by traffic-weighted cascade assignment (net of escalation) across the full task taxonomy. Cells are measured on the logged share of sessions and extrapolated to unlogged traffic under a same-distribution assumption, spot-checked for bias; new or unmatched cells default to the mid-tier incumbent (Sonnet~5) until benchmarked. The full candidate roster additionally includes Haiku for trivial or latency-insensitive cells (e.g.\ greetings, subagent calls, background jobs) not shown in Table~\ref{tab:cascade}; Sonnet~4.6's aggregated share is a small residual pending final benchmarking and is not broken out separately here. Percentages are illustrative until benchmarking completes; the method, not the numbers, is what is under review.}
\label{tab:targetmix}
\end{table}

Every cell assignment should carry an explicit confidence tier: 
\emph{Confirmed} if the sample size is adequate, the result is verifier-graded, and it is stable across repeated benchmark runs; 
\emph{Leaning} if evidence is directional but incomplete; and 
\emph{Not yet benchmarked} if the incumbent remains the default.

Several displacement opportunities do not require difficulty prediction. For example, subagent traffic can often be routed to smaller models; initial greetings and session starters should not use frontier models; and cron or background jobs are latency-insensitive and often suitable for owned-GPU open-source serving.

\section{Hierarchical Decomposition and the Savings Waterfall}
\label{sec:hierarchy}
The optimization above operates on cells, but a cell is not an arbitrary bucket: it is a leaf of a two-level hierarchy that mirrors how the traffic is actually generated. At the top, the coding-assistant surface accounts for a fraction $a$ of all enterprise LLM traffic. Within it, taxonomy discovery (Section~\ref{sec:benchmark}) yields a set of root task categories $k\in\mathcal{K}$---in our environment the largest by cost-weighted share are \emph{software development} ($24.6\%$), \emph{AI, tooling \& documentation} ($13.8\%$), \emph{testing \& QA} ($9.8\%$), \emph{code review \& version control} ($8.4\%$), and \emph{debugging \& root-cause analysis} ($8.2\%$)---with cost-weighted shares $\pi_k$, where $\sum_k \pi_k = 1$. Difficulty grading then splits each subcategory into tiers $d\in\{\textsf{easy},\textsf{medium},\textsf{difficult}\}$ with within-subcategory shares $\rho_{d\mid k}$, where $\sum_d \rho_{d\mid k}=1$. A cell is the pair $c=(k,d)$, and its share of coding-assistant traffic is
\begin{equation}
w_{k,d} \;=\; a\,\cdot\,\pi_k\,\cdot\,\rho_{d\mid k}.
\label{eq:cellweight}
\end{equation}
This decomposition is the reason routing is decided per cell: difficulty varies sharply \emph{within} a category, so a single per-category model choice would either over-provision easy tasks to a frontier model or under-provision hard tasks to a weak one. Figure~\ref{fig:hierarchy} traces the decomposition for two benchmarked subcategories whose measured GLM~5.2 pass rates lead to \emph{opposite} displacement decisions.

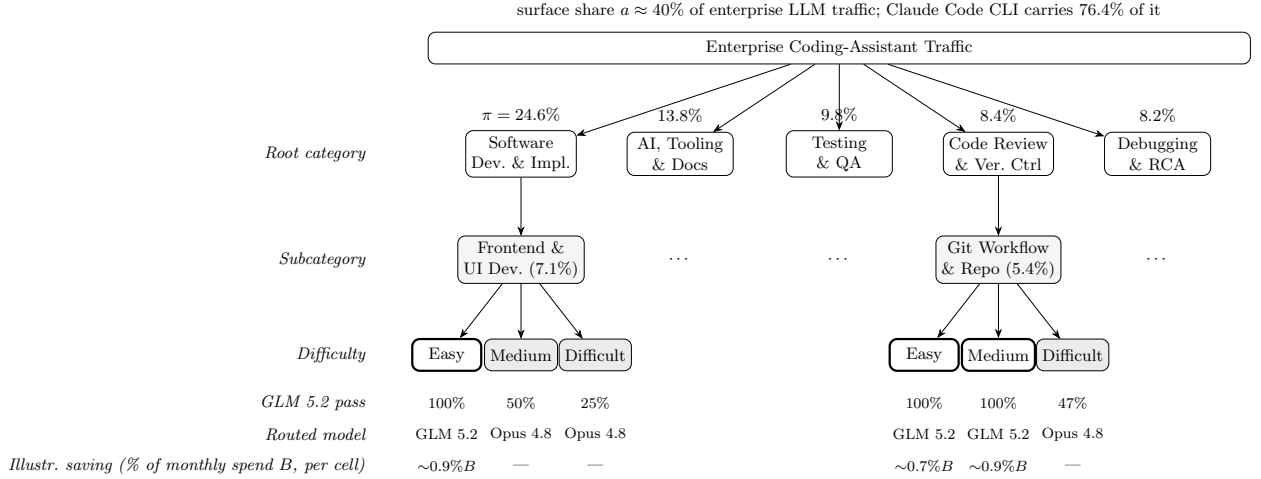
\begin{figure*}[t]
\centering
\resizebox{\textwidth}{!}{%
\begin{tikzpicture}[
  font=\footnotesize,
  box/.style={draw, rounded corners, minimum height=6mm, align=center, inner sep=3pt},
  cat/.style={box, minimum width=20mm, align=center},
  sub/.style={box, minimum width=22mm, align=center, fill=black!4},
  lvl/.style={box, minimum width=13mm},
  keep/.style={box, minimum width=13mm, fill=black!8},
  disp/.style={box, minimum width=13mm, very thick},
  rowlab/.style={anchor=east, font=\footnotesize\itshape},
  >=Stealth
]
\node[box, minimum width=155mm] (traffic) at (0,7) {Enterprise Coding-Assistant Traffic};
\node[above=1.2mm of traffic, font=\footnotesize] {surface share $a\approx 40\%$ of enterprise LLM traffic; Claude Code CLI carries $76.4\%$ of it};

\node[cat] (sw)   at (-6,5) {Software\\Dev.\ \& Impl.};
\node[cat] (ai)   at (-3,5) {AI, Tooling\\\& Docs};
\node[cat] (test) at (0,5)  {Testing\\\& QA};
\node[cat] (rev)  at (3,5)  {Code Review\\\& Ver.\ Ctrl};
\node[cat] (dbg)  at (6,5)  {Debugging\\\& RCA};
\node[above=0.5mm of sw]   {$\pi=24.6\%$};
\node[above=0.5mm of ai]   {$13.8\%$};
\node[above=0.5mm of test] {$9.8\%$};
\node[above=0.5mm of rev]  {$8.4\%$};
\node[above=0.5mm of dbg]  {$8.2\%$};
\foreach \n in {sw,ai,test,rev,dbg}{\draw[->] (traffic) -- (\n);}

\node[sub] (fe)  at (-6,3) {Frontend \&\\UI Dev.\ (7.1\%)};
\node[sub] (git) at (3,3)  {Git Workflow\\\& Repo (5.4\%)};
\draw[->] (sw)  -- (fe);
\draw[->] (rev) -- (git);
\node at (-3,3) {$\cdots$}; \node at (0,3) {$\cdots$}; \node at (6,3) {$\cdots$};

\node[disp] (fE) at (-7.4,1.2) {Easy};
\node[keep] (fM) at (-6,1.2)   {Medium};
\node[keep] (fD) at (-4.6,1.2) {Difficult};
\draw[->] (fe) -- (fE); \draw[->] (fe) -- (fM); \draw[->] (fe) -- (fD);

\node[disp] (gE) at (1.6,1.2) {Easy};
\node[disp] (gM) at (3,1.2)   {Medium};
\node[keep] (gD) at (4.4,1.2) {Difficult};
\draw[->] (git) -- (gE); \draw[->] (git) -- (gM); \draw[->] (git) -- (gD);

\node[font=\scriptsize] at (-7.4,0.3) {100\%}; \node[font=\scriptsize] at (-6,0.3) {50\%};      \node[font=\scriptsize] at (-4.6,0.3) {25\%};
\node[font=\scriptsize] at (-7.4,-0.3){GLM 5.2}; \node[font=\scriptsize] at (-6,-0.3){Opus 4.8}; \node[font=\scriptsize] at (-4.6,-0.3){Opus 4.8};
\node[font=\scriptsize] at (-7.4,-0.9){$\sim$0.9\%$B$};  \node[font=\scriptsize] at (-6,-0.9){---};       \node[font=\scriptsize] at (-4.6,-0.9){---};

\node[font=\scriptsize] at (1.6,0.3) {100\%};  \node[font=\scriptsize] at (3,0.3) {100\%};     \node[font=\scriptsize] at (4.4,0.3) {47\%};
\node[font=\scriptsize] at (1.6,-0.3){GLM 5.2}; \node[font=\scriptsize] at (3,-0.3){GLM 5.2};   \node[font=\scriptsize] at (4.4,-0.3){Opus 4.8};
\node[font=\scriptsize] at (1.6,-0.9){$\sim$0.7\%$B$};  \node[font=\scriptsize] at (3,-0.9){$\sim$0.9\%$B$};      \node[font=\scriptsize] at (4.4,-0.9){---};

\node[rowlab] at (-8.8,5)   {Root category};
\node[rowlab] at (-8.8,3)   {Subcategory};
\node[rowlab] at (-8.8,1.2) {Difficulty};
\node[rowlab] at (-8.8,0.3) {GLM 5.2 pass};
\node[rowlab] at (-8.8,-0.3){Routed model};
\node[rowlab] at (-8.8,-0.9){Illustr.\ saving (\% of monthly spend $B$, per cell)};
\end{tikzpicture}%
}
\caption{Hierarchical decomposition of coding-assistant traffic into task-category $\times$ difficulty cells. Root-category and subcategory shares (\%) are measured cost-weighted shares from production taxonomy discovery; the \emph{GLM 5.2 pass} rates are measured in-harness on graded tasks (Git Workflow: $100/100/47\%$ for easy/medium/difficult; Frontend \& UI: $100/50/25\%$). The two subcategories carry near-identical traffic yet reach \emph{opposite} decisions: GLM 5.2 clears the routing boundary (Eq.~\eqref{eq:boundary}) on the easy--medium Git cells (bold, displaced) but only the easy Frontend cell, so the harder Frontend cells stay on the frontier incumbent (shaded, kept). Per-cell savings figures are \emph{illustrative}, expressed as a percentage of total monthly coding-assistant spend $B$ (not a company-wide dollar total) that assume an illustrative within-subcategory difficulty split; contracted per-token prices and absolute spend figures are withheld.}
\label{fig:hierarchy}
\end{figure*}

Once each displaced cell $c$ has a chosen candidate $m^\star(c)$ and an incumbent $M$, the addressable saving is a traffic-weighted sum of per-task savings across the displaced cells:
\begin{equation}
S \;=\; N \sum_{(k,d)\,\in\,\mathcal{D}} w_{k,d}\,\bigl(\,C_M(k,d) - \mathbb{E}[\mathrm{cost}\mid m^\star,(k,d)]\,\bigr),
\label{eq:waterfall}
\end{equation}
where $N$ is the number of completed tasks per period and $\mathcal{D}$ is the set of cells that clear the routing boundary of Eq.~\eqref{eq:boundary}. Ordering the cells by their contribution to Eq.~\eqref{eq:waterfall} yields a \emph{savings waterfall}: an auditable, dollar-ranked backlog of displacement opportunities that directs benchmark and rollout effort to where the money is. Because each term is gated by Proposition~\ref{prop:dominance}, every entry in the waterfall weakly reduces realized end-to-end cost, not just token cost.

Table~\ref{tab:worked} instantiates the two benchmarked subcategories from Figure~\ref{fig:hierarchy} with measured data. The subcategory cost shares (Frontend \& UI Development at $7.1\%$ and Git Workflow \& Repository Management at $5.4\%$ of total spend) and the per-difficulty GLM~5.2 pass rates are measured---the latter in the same agentic harness developers use---while the within-subcategory difficulty split $\rho_{d\mid k}$ is an illustrative placeholder pending the difficulty grader's per-cell output. The two subcategories carry near-identical traffic yet reach opposite conclusions: GLM~5.2 passes $100\%$ of easy \emph{and} medium Git-workflow tasks, so both cells displace to the mid tier, but its accuracy on Frontend work collapses from $100\%$ (easy) to $50\%$ (medium) to $25\%$ (difficult), so only the easy Frontend cell displaces and the rest stay on the frontier incumbent. This is exactly the per-cell behavior Eq.~\eqref{eq:cellweight} anticipates: displacement is a property of the (subcategory, difficulty) cell, not of the category.

For illustration, the per-cell savings in Figure~\ref{fig:hierarchy} and Table~\ref{tab:worked} are computed proportionally as $S_{k,d}=B\,\pi_k\,\rho_{d\mid k}\,\delta_{\mathrm{tier}}$ over the displaced cells only, where $B$ is the monthly coding-assistant spend, $\pi_k$ the subcategory's measured cost share, $\rho_{d\mid k}$ the illustrative within-subcategory difficulty share, and $\delta_{\mathrm{tier}}=0.35$ the net displacement fraction of a frontier cell's spend recovered by the GLM mid tier (net of escalation). Kept cells contribute zero. Consistent with the confidence tiers of Section~\ref{sec:mix}, only cells backed by adequate, verifier-graded, stable benchmarks are treated as \emph{Confirmed} displacements; subcategories without confirmed coverage are omitted rather than estimated. Candidate models are drawn from the internal model registry and named by capability tier; their contracted prices are deliberately omitted.

The failure modes behind these pass rates are consistent and diagnosable rather than random noise. On golden replay tasks, GLM~5.2 mutates repository state before inspecting an unexpected submodule modification (whereas the frontier model inspects first and preserves local work); it stops at a single database constraint instead of tracing a question across the persistence, service, and UI layers; and it re-inspects an already-completed build instead of advancing to the next workflow step. Each corresponds to a routing signal---unexpected mutable state, cross-layer reasoning, and multi-step state tracking---that our five-dimension difficulty check (scope, next-step clarity, evidence clarity, action risk, and judgment required) uses to escalate borderline cells to the frontier tier.

\begin{table}[t]
\centering
\footnotesize
\begin{tabular}{l l r r l r}
\toprule
\textbf{Subcategory} & \textbf{Diff.} & \textbf{GLM pass} & \textbf{Share} & \textbf{Routed model} & \textbf{Saving (\% of $B$/mo)} \\
\midrule
Frontend \& UI Dev.\ (7.1\%)   & Easy      & 100\% & 35\% & GLM 5.2 (mid)  & $\sim$0.9\% \\
Frontend \& UI Dev.\ (7.1\%)   & Medium    & 50\%  & 45\% & Opus 4.8 (keep) & ---  \\
Frontend \& UI Dev.\ (7.1\%)   & Difficult & 25\%  & 20\% & Opus 4.8 (keep) & ---  \\
\addlinespace
Git Workflow \& Repo (5.4\%)    & Easy      & 100\% & 35\% & GLM 5.2 (mid)  & $\sim$0.7\% \\
Git Workflow \& Repo (5.4\%)    & Medium    & 100\% & 45\% & GLM 5.2 (mid)  & $\sim$0.9\% \\
Git Workflow \& Repo (5.4\%)    & Difficult & 47\%  & 20\% & Opus 4.8 (keep) & ---  \\
\bottomrule
\end{tabular}
\caption{Worked savings waterfall for two benchmarked subcategories (cf.\ Figure~\ref{fig:hierarchy}), out of the full task taxonomy that spans monthly spend in the millions of dollars. Subcategory cost shares and the ``GLM pass'' rates are measured (production taxonomy and in-harness GLM 5.2 grading); the within-subcategory difficulty ``Share'' $\rho_{d\mid k}$ is an illustrative split pending the difficulty grader's per-cell output. A cell is displaced to GLM 5.2 only when its measured pass rate clears the routing boundary (Eq.~\eqref{eq:boundary}); otherwise it is kept on the frontier incumbent. ``Saving'' is an illustrative \emph{per-cell} monthly amount, expressed as a percentage of total monthly spend $B$, computed via Eq.~\eqref{eq:waterfall} with a mid-tier displacement fraction $\delta=0.35$; contracted per-token prices and absolute spend figures are withheld. It is not a company-wide total: only the two subcategories with confirmed benchmark coverage are shown, and every other cell in the taxonomy is omitted rather than estimated.}
\label{tab:worked}
\end{table}

\section{Forecasting and Version Planning}
The forecast decomposes monthly spend as a function of developers, sessions per developer, traffic share by cell, model mix share, token consumption per task, and model price. This decomposition separates the effects of adoption, activity, mix, and price. Version and replacement planning treats model releases as scheduled scenario events on a rolling twelve-month horizon.

There are three event classes. First, same-family version upgrades may reclaim work from previously displaced cells, so they require re-benchmarking and repricing. Second, open-source replacements may arrive on a faster but noisier cadence, requiring legal and export-control review before deployment eligibility. Third, pricing events such as contractual discount expiries must be represented as named scenarios with dates, not surprises.

Because exact release timing is uncertain, forecasts should be published as scenario bands: a base case with announced and contracted events, an upside case where open-source displacement ramps at benchmark-confirmed pace, and a downside case where releases slip or discount expiry arrives without offsetting displacement. Monthly reconciliation should attribute variance to adoption, activity, mix, or price.

\section{From Recommendations to an Intelligent Router}
\label{sec:router}
The routing capability matures through four stages. Stage one is a static policy plus developer guidance matrix with confidence tiers. Stage two is a shadow-mode classifier that predicts the route but does not act. Stage three is a verified cascade that routes to the cheapest eligible model first in categories with cheap verifiers and silently escalates on failure. Stage four is a full intelligent router that combines classifier and cascade.

The cascade is especially important because it removes the need for ex ante difficulty prediction in many categories. If a candidate attempt fails, the system can escalate automatically before the developer notices. This makes cascade routing a pragmatic bridge between static policy and full prediction.

To protect developer trust, the system must include override controls and kill switches from day one. Developers should be able to pin a model per session or per repository, and any subcategory-by-difficulty cell should be able to revert to the incumbent within minutes. Quality telemetry beyond pass rate---such as PR revert rate, reprompt rate, abandonment, and review cycle time---should be monitored continuously.

\section{Governance and Operating Cadence}
\method{} is a continuously operating system, not a one-time study. Traffic drifts as new frameworks appear and subagent usage grows; routing itself changes developer behavior; and models and prices change on a quarterly or faster cadence. Continuous telemetry ingestion and guardrail monitoring should feed back into labels. Monthly operations should include taxonomy reclustering, forecast reconciliation, and leadership review. Quarterly operations should include judge recalibration, benchmark refresh, and routing-policy audit. Event-driven operations should trigger re-benchmarking after model releases or pricing changes before any routing change is allowed to ship.

The primary program KPIs are blended cost per completed task, displaced share of spend, realized versus forecast savings, escalation rate in cascade cells, quality guardrails relative to baseline, unclassified traffic share, and benchmark coverage.

\section{Key Risks and Mitigations}
The main risks are well-defined. First, cost-per-token optimization can backfire if retries and escalations exceed the savings from cheaper tokens; the mitigation is to optimize cost per completed task with escalation priced in. Second, judge bias or drift can corrupt labels; the mitigation is human calibration, neutral judging, verifier-first grading, and quarterly recalibration. Third, benchmark results may not transfer to production; the mitigation is in-harness evaluation, distribution-matched sampling, and validation with production outcome signals. Fourth, ex ante difficulty prediction may be unreliable; the mitigation is cascade routing in verifier-rich categories and shadow-mode evaluation before any classifier is allowed to act. Fifth, taxonomy drift may invalidate the mix; the mitigation is monthly reclustering, stable identifiers, and an unclassified-share KPI. Sixth, new model releases or discount expiries may change the economics; the mitigation is event-driven re-benchmarking and scenario-based forecasting. Seventh, developer trust may erode from silent quality degradation; the mitigation is developer override, kill switches, and richer quality telemetry. Finally, legal or compliance issues may block open-source deployment; the mitigation is to treat legal and export-control review as an explicit eligibility gate.

\section{Discussion}
\method{} shifts routing from token-centric heuristics to task-centric optimization. The framework is practical because it is rooted in telemetry, benchmarked on production-like tasks, and deployable through staged routing. It is also extensible: the same structure supports task classes beyond coding, alternative benchmark sources, and richer cost models that include human review or downstream defect costs.

\method{} is best viewed as a systems framework rather than a single algorithm. Its novelty lies in combining taxonomy discovery, difficulty-aware evaluation, expected-cost routing, and governance under drift. The result is less a better router than an operating model for continuously reallocating work across heterogeneous language models as capabilities, prices, and workloads evolve.

\section{Limitations}
This draft is grounded in a single enterprise setting and therefore does not yet establish broad external validity. The strongest empirical claims require controlled experiments across multiple task domains, multiple model families, and longer time horizons. In addition, parts of the framework depend on reliable verifiers; categories without cheap automatic verifiers will rely more heavily on judges and may be harder to route safely. Finally, this draft describes a methodology and a system architecture, but it does not yet include a full ablation study or a live A/B result set.

\section{Conclusion}
We presented \method{}, a task-to-model optimization framework for enterprise AI coding assistants. The method discovers a task taxonomy from production traffic, grades difficulty, constructs distribution-matched benchmarks, evaluates candidate models in a production harness, and solves for the cost-optimal model mix under quality and latency constraints. Its objective is cost per completed task, not cost per token, and its deployment path progresses from static policy to shadow mode, verified cascade, and eventually an intelligent router. We believe this framing provides a practical and publishable foundation for future work on large-scale LLM routing in enterprise settings.


\begin{thebibliography}{9}

\bibitem{clio}
Tamkin, A., McCain, M., Handa, K., et al.
\newblock Clio: Privacy-Preserving Insights into Real-World AI Use.
\newblock arXiv:2412.13678, 2024.

\bibitem{frugalgpt}
Chen, L., Zaharia, M., Zou, J.
\newblock FrugalGPT: How to Use Large Language Models While Reducing Cost and Improving Performance.
\newblock arXiv:2305.05176, 2023.

\bibitem{routellm}
Ong, I., Almahairi, A., Wu, V., Zhang, W., Willmott, D., Ang, S., Sistla, R., Gonzalez, J.E., Stoica, I.
\newblock RouteLLM: Learning to Route LLMs with Preference Data.
\newblock arXiv:2406.18665, 2024.

\bibitem{hybridllm}
Ding, D., Mallick, A., Wang, C., Sim, R., Mukherjee, S., Ruhle, V., Lakshmanan, L.V.S., Awadallah, A.
\newblock Hybrid LLM: Cost-Efficient and Quality-Aware Query Routing.
\newblock ICLR; arXiv:2404.14618, 2024.

\bibitem{automix}
Madaan, A., Aggarwal, P., Anand, A., et al.
\newblock AutoMix: Automatically Mixing Language Models.
\newblock arXiv:2310.12963, 2023.

\bibitem{shnitzer}
Shnitzer, T., Ou, A., Silva, M., Soule, K., Sun, Y., Solomon, J., Thompson, N., Yurochkin, M.
\newblock Large Language Model Routing with Benchmark Datasets.
\newblock arXiv:2309.15789, 2023.

\bibitem{swebench}
Jimenez, C.E., Yang, J., Wettig, A., Yao, S., Pei, K., Press, O., Narasimhan, K.
\newblock SWE-bench: Can Language Models Resolve Real-World GitHub Issues?
\newblock ICLR; arXiv:2310.06770, 2024.

\bibitem{humaneval}
Chen, M., Tworek, J., Jun, H., et al.
\newblock Evaluating Large Language Models Trained on Code.
\newblock arXiv:2107.03374, 2021.

\bibitem{mtbench}
Zheng, L., Chiang, W.-L., Sheng, Y., et al.
\newblock Judging LLM-as-a-Judge with MT-Bench and Chatbot Arena.
\newblock NeurIPS; arXiv:2306.05685, 2023.

\end{thebibliography}
\end{document}